\documentclass[conference]{IEEEtran}

\usepackage{amsmath}
\usepackage{amsthm}
\usepackage{amsfonts}
\usepackage{amssymb}
\usepackage{makecell}

\usepackage{booktabs}
\usepackage{etoolbox}
\usepackage{siunitx}
\usepackage{multirow}
\usepackage{xcolor}
\usepackage{comment}
\newtheorem{definition}{Definition}
\newtheorem{proposition}{Proposition}

\DeclareMathOperator*{\argmin}{argmin}
\DeclareMathOperator*{\argmax}{argmax}
\DeclareMathOperator{\diam}{D}
\DeclareMathOperator{\radi}{R}
\DeclareMathOperator{\centerD}{D-center}
\DeclareMathOperator{\centerR}{R-center}
\DeclareMathOperator{\intracldist}{distance}
\newcommand{\rmax}{\ensuremath{R_{max}}}
\newcommand{\dmax}{\ensuremath{D_{max}}}
\newcommand{\thr}{\ensuremath{T}}

\newcolumntype{R}{@{\extracolsep{\fill}}r@{\extracolsep{0pt}}}

\makeatletter
\renewenvironment{proof}[1][\proofname]{\par
  \pushQED{\qed}%
  \normalfont \topsep6\p@\@plus6\p@\relax
  \trivlist
  \item\relax
      {\itshape
    #1\@addpunct{.}}\hspace\labelsep\ignorespaces
}{%
  \popQED\endtrivlist\@endpefalse
}
\makeatother

\begin{document}

\title{Clustering to the Fewest Clusters Under Intra-Cluster Dissimilarity Constraints}

\author{\IEEEauthorblockN{Jennie Andersen}
\IEEEauthorblockA{University of Lyon, INSA Lyon, CNRS UMR 5205\\
jennie.andersen@liris.cnrs.fr}
\and
\IEEEauthorblockN{Brice Chardin}
\IEEEauthorblockA{ISAE-ENSMA, LIAS\\
brice.chardin@ensma.fr}
\and
\IEEEauthorblockN{Mohamed Tribak}
\IEEEauthorblockA{SRD, LIAS\\
mohamed.tribak@srd-energies.fr}
}

\maketitle

\begin{abstract}
This paper introduces the equiwide clustering problem, where valid partitions must satisfy intra-cluster dissimilarity constraints.
Unlike most existing clustering algorithms, equiwide clustering relies neither on density nor on a predefined number of expected classes, but on a dissimilarity threshold.
Its main goal is to ensure an upper bound on the error induced by ultimately replacing any object with its cluster representative.
Under this constraint, we then primarily focus on minimizing the number of clusters, along with potential sub-objectives.

We argue that equiwide clustering is a sound clustering problem, and discuss its relationship with other optimization problems, existing and novel implementations as well as approximation strategies.
We review and evaluate suitable clustering algorithms to identify trade-offs between the various practical solutions for this clustering problem.
\end{abstract}

\section{Introduction}

\label{sec:introduction}

When the expected number of classes is unknown, there exists no universal method to determine the number of clusters \cite{milligan1985}.
To circumvent this problem, several algorithms instead rely on density to detect groups of elements.
However, density-based clustering can lead to elongated clusters, which allows for relatively dissimilar elements to belong to the same cluster.

In this work, we are concerned with clustering problems for which the number of clusters is to be determined, with strong guaranties on the dissimilarity of elements belonging to the same cluster.
We therefore introduce the equiwide clustering problem, where valid partitions must satisfy intra-cluster dissimilarity constraints.
Unlike most existing clustering algorithms, equiwide clustering relies neither on density nor on a predefined number of expected classes, but on a dissimilarity threshold.
Its main goal is to ensure an upper bound on the error induced by ultimately replacing any object with its cluster representative.
Under this constraint, we then primarily focus on minimizing the number of clusters, and discuss potential sub-objectives.
Equiwide clustering aims at solving problems for which the dissimilarity measure can be reasoned upon by domain experts.

\subsubsection*{Industrial use case}

SRD is an electricity distribution network operator that manages its network on a regional level over a surface of 7000~km\textsuperscript{2}, with 12,323 km of network lines, 16 step-down substations and more than 8000 public distribution substations. To optimize the topology of its network, SRD models the behavior of its step-down substations, that connect the national transmission network (with voltages ranging from 63~kV to 400~kV) to its local distribution network (with a voltage of 20~kV). Since there are 180 substations output power lines, modeling each of them individually is impractical. A first step consists in grouping them into categories to then define a common model.

Yet, what categories exist is unknown. In particular, their number is to be determined. Network operators at SRD are only able to specify a similarity measure and a threshold under which two substation outputs could be considered as belonging to the same category.
The goal is, hopefully, to identify the fewest number of categories in order to simplify subsequent modeling work in the network optimization process.

\medskip

Other practical applications of equiwide clustering are taken from related optimization problems (presented in Section~\ref{sec:related}), such as scheduling or radio frequency assignment \cite{johnson1996}.
In this paper, we discuss this kind of clustering problem and solutions from the state of the art.

\subsubsection*{Paper organization}

In the remainder of this section, we provide preliminary definitions and formalize the equiwide clustering problem.
In Section~\ref{sec:related}, we discuss its relationship with other optimization problems, along with other approaches from the state of the art.
In Section~\ref{sec:eqw}, we consider various implementation strategies and approximations.
We then provide an experimental evaluation of relevant algorithms in Section~\ref{sec:experiments}.

\subsection{Preliminary definitions}
\label{sec:preliminaries}

\begin{definition}[Population]
Let $X = \left\{x_1, ..., x_n\right\}$ be a finite set of $n$ elements to be partitioned.
\end{definition}

\begin{definition}[Partition]
Let $P = \left\{C_1, ..., C_k\right\}$ be a partition of $X$:
\begin{align*}
    & \textstyle\bigcup_{C_i \in P} C_i = X && \tag{cover} \\
    & \forall C_i \in P, C_i \neq \emptyset && \tag{non-emptiness}\\
    & \forall C_i, C_j \in P, i \neq j \Rightarrow C_i \cap C_j = \emptyset && \tag{pairwise disjunction}
\end{align*}
\end{definition}

\begin{definition}[Dissimilarity]
\label{def:dissimilarity}
Let $d$ be a dissimilarity on pairs of elements of $X$, $\forall a, b \in X$:
\begin{align*}
    & d(a, a) = 0 \\
    & d(a, b) = d(b, a) \tag{symmetry} \\
    & d(a, b) \geq 0 \tag{positivity}
\end{align*}
\end{definition}

We consider two concepts of wideness to characterize the homogeneity of a subset of elements of $X$, its diameter and its radius.

\begin{definition}[Diameter]
The diameter of a subset $A$ of $X$ is the maximum dissimilarity between pairs of elements of $A$.
$$\diam(A) = \max_{a, b \in A} d(a, b)$$
\end{definition}

\begin{definition}[Radius]
The radius of a subset $A$ of $X$, first defined as a measure of homogeneity of a cluster \cite{hansen1997cluster}, is the minimum eccentricity of any element of $A$.
$$\radi(A) = \min_{a \in A} \max_{b \in A} d(a, b)$$
\end{definition}

\begin{definition}[Homogeneity]
A subset $A$ of $X$ is called homogeneous (specifically diameter-homogeneous or radius-homogeneous) under a provided threshold $\thr$ if its wideness (diameter or radius) is at most equal to this threshold. In the remainder of this paper, a diameter-specific threshold will be noted $\dmax$ and a radius-specific threshold will be noted $\rmax$.

A partition $P$ of $X$ is called homogeneous if each set in $P$ is homogeneous.
\end{definition}

In the next two definitions, we are concerned with the identification of a representative element for a subset $A$ of $X$. This element, called the \textit{center} of $A$, has the smallest maximal dissimilarity with every element of $A$.

\begin{definition}[Center, diameter -- normed vector space]
If $X$ is a subset of a normed vector space, the center is the midpoint of the two furthest elements in $A$, i.e., the pair of elements whose dissimilarity is equal to the diameter. This center is not necessarily an element of $A$ itself.

$$\centerD(A) = \frac{a + b}{2}, \quad (a, b) = \argmax_{(i, j) \in A \times A}d(i, j)$$
\end{definition}

\begin{definition}[Center, radius]
The central element of $A$ is the element whose greatest dissimilarity with every element of $A$ is minimal. This center is not necessarily the medoid, as a medoid minimizes the average dissimilarity.

$$\centerR(A) = \argmin_{a \in A} \underset{b \in A}{\max\vphantom{g}} d(a, b)$$
\end{definition}

\subsection{Problem statement}

\begin{definition}[Equiwide clustering]
\label{def:problem_statement}
For a given population $X$, a dissimilarity $d$ and a threshold $\thr$, we are concerned with the identification of a partition $P$ of $X$ so that:

\begin{itemize}
    \item $P$ is homogeneous under $\thr$,
    \item the number of clusters $|P|$ is minimal, i.e., there exists no partition $P'$ homogeneous under $\thr$ where $|P'| < |P|$.
\end{itemize}
\end{definition}

When we are concerned with the identification of a representative element for each cluster with an upper-bound on its dissimilarity with every element of the cluster, a diameter constraint will not provide a tight upper bound if a center (diameter) cannot be computed. This problem notably occurs when the elements to be partitioned are not defined within a normed vector space.
In that case, a radius constraint offers tighter dissimilarity guaranties between a representative element---the center (radius) of each cluster---and every other element of the cluster.
While conceptually similar, these two types of wideness constraints (diameter or radius) introduce significant variations on how the problem can be resolved efficiently.

\subsubsection*{Sub-objectives}

This clustering problem may have multiple solutions with the same number of clusters, and sub-objectives can be added to discriminate between them. These sub-objectives may vary depending on the data to be partitioned. In this paper, we consider three sub-objectives.

\paragraph*{Minimizing the maximum width}

This subgoal aims at providing tighter upper bounds on the intra-cluster dissimilarity constraint considered---either the diameter or the radius.

$$\text{minimize} \max_{C_i \in P}\diam(C_i) \text{ \quad or \quad minimize} \max_{C_i \in P}\radi(C_i)$$

\paragraph*{Minimizing the within-cluster sum of dissimilarities}

This global homogeneity measure of a partition has been considered as a sub-objective in related work \cite{duong2017constrained}.

$$ \text{minimize} \sum_{C_i \in P} \sum_{a, b \in C_i} d(a,b) $$

\paragraph*{Maximizing the variance of cluster sizes}

In the context of SRD described in Section~\ref{sec:introduction}, experts prefer to obtain few large clusters and some outliers, rather than clusters of average size.
The variance on cluster sizes captures this criterion and has therefore been considered in this study.

$$\text{maximize} \sum_{C_i \in P}|C_i|^2$$

\section{Related Work}

\label{sec:related}

\subsection{Equivalent problems}

\label{sec:related-problems}

The equiwide clustering problem---as defined in definition \ref{def:problem_statement}, without sub-obje\-ctives---has been considered in graph theory. Some theoretical results therefore apply, as well as algorithms to compute partitions with a minimal number of clusters. Yet, since sub-objectives are absent from these formulations, resulting partitions may not be entirely satisfactory.

Let $G = (X, E)$ be a graph where $E = \{\{x_i, x_j\} \mid d(x_i, x_j) \leq \thr\}$, namely there is an edge between vertices if and only if the dissimilarity between them is at most equal to the provided threshold, i.e., adjacent vertices are compatible.

\paragraph*{Minimum clique cover} The diameter-based clustering problem is equivalent to the minimum (vertex) clique cover problem in graph theory. Each resulting clique represents a cluster, within which each pair of element is compatible, i.e., their dissimilarity is at most equal to the diameter. This problem is NP-hard \cite{karp1972}.

\paragraph*{Graph coloring} Let $G'$ be the complement graph of $G$, i.e., vertices of $G'$ are adjacent if and only if they are incompatible.
The minimum clique cover on graph $G$ is equivalent to coloring vertices of $G'$ using a minimal number of colors. The resulting colors represent clusters: vertices of the same color belong to the same cluster.

\paragraph*{Minimal cardinality dominating set} A dominating set of $G$ is a subset $D$ of $X$ so that every vertex of $X \setminus D$ is adjacent to at least one vertex of $D$. The radius-based clustering problem is equivalent to the minimal cardinality dominating set problem in graph theory. Each dominating element is the center (radius) of a cluster. This problem is NP-hard \cite{garey1979}.

\subsection{Related clustering algorithms}

Only a few existing clustering techniques meet, at least partially, the requirements of equiwide clustering. For instance, the well-known k-means algorithm requires the number of clusters to be provided and cannot respect a dissimilarity-based constraint.
A constraint on a maximal diameter can be added by means of instance-level cannot-link constraints (COP-Kmeans) \cite{copkmeans}, but the number of clusters is still required.
While density-based clustering algorithms can be used to identify the number of clusters, they do not enforce wideness constraints on the resulting partition, and are therefore not considered in this paper.

\subsubsection{Hierarchical clustering}
Hierarchical clustering techniques build a hierarchy of clusters that can later be used to create a partition. A cut-off criterion can be defined, either as a number of clusters or as a dissimilarity threshold. The complete-link agglomerative hierarchical clustering is a fast algorithm on which a maximum diameter constraint can be set.
The complete-link criterion merges two clusters $A$ and $B$ based on the maximum dissimilarity between pairs of elements of each cluster.
$$\intracldist\left(A,B\right) = \max_{a \in A, b\in B}d\left(a,b\right)$$

If the diameter threshold is applied on this distance, the maximum diameter criterion is respected.
While the complete-link agglomerative hierarchical clustering with a cut-off distance is a valid equiwide clustering method with respect to its main constraint (wideness -- diameter only), it does not provide any optimality guarantee. In particular, this algorithm does not attempt to minimize the number of clusters. It is also not applicable to a radius constraint, apart from considering the pairwise dissimilarity as a loose upper bound.

\subsubsection{Graph coloring}
Graph coloring algorithms that minimize the number of colors can be used to cluster elements with a maximum diameter on the associated graph (see Section~\ref{sec:related-problems}).
Hansen and Delattre \cite{hansen1978complete} proposed an algorithm, called \textsc{clustergraph}, that also minimizes the maximum diameter of clusters.
This algorithm relies on an external optimal graph coloring algorithm.
While this algorithm provides a partition that is optimal on both the main objective---minimizing the number of clusters---and the maximum diameter sub-objective, it is also not usable with a radius constraint.

\subsubsection{Constraint programming}
A constraint programming algorithm has been proposed by Dao et al. \cite{duong2017constrained} for constrained clustering. Among the existing constraints, it is possible to specify a maximum diameter. A few objective functions have been defined, including minimizing the maximum diameter or minimizing the within-cluster sum of dissimilarities. With their algorithm, it is possible to find an optimal solution for these objectives or to stop at the first solution satisfying the maximum diameter constraint.
While their proposition does not directly minimize the number of clusters---this value is an input of the algorithm---a partition with a minimum number of clusters can be identified by iterating over the number of clusters until a valid solution is found.

In order to improve the run time of the algorithm, the authors suggest to reorder elements according to a Furthest-Point-First (FPF) heuristic: the first element is the one furthest from all the other elements, then the following elements are ordered by decreasing dissimilarity with all previously identified furthest points.
This heuristic is used to guide the search strategy, but also to provide a lower and upper bound on the optimal diameter during the optimization process.

\subsubsection{Cannot-link constraints}
An intra-cluster dissimilarity constraint is a cluster-level constraint, but it can also be expressed as an instance-level cannot-link constraint. Two elements affected by a cannot-link constraint cannot belong to the same cluster. Consequently, a maximum diameter constraint can be translated into multiple cannot-link constraints between each pair of elements whose dissimilarity is greater than $D_{max}$.
These instance-level constraints have been considered to augment existing clustering algorithms such as DBSCAN \cite{cdbscan} and k-means \cite{copkmeans}.

\section{Equiwide clustering}

\label{sec:eqw}

\subsection{Constraint programming and linear programming formulations}
The problem of minimizing the number of clusters with respect to a maximum radius or diameter can be formulated as a constrained optimization problem.
This formulation allows for solvers to be used as a baseline during the evaluation of the various implementations.

\subsubsection*{Radius constraint}
Let $r$ be an $n \times n$ Boolean matrix. The value $r_{ij}$ represents whether element $i$ belongs to the cluster centered on element $j$.
The value $r_{jj}$ indicates if element $j$ is the \textit{center} of a cluster.
\begin{subequations}
    \label{eqn:lp-radius}
    \begin{alignat}{3}
        & \text{minimize }   &       & \sum_j r_{jj} \label{eqn:radius-objective}\\
        & \text{subject to } & \quad & \sum_j r_{ij} = 1                 & \quad & i \in \{1, \dots, n\}   \label{eqn:radius-unique}\\
        &                    &       & d(i,j) \times r_{ij} \leq R_{max} & \quad & i,j \in \{1, \dots, n\} \label{eqn:radius-threshold}\\
        &                    &       & r_{jj} = \max_i(r_{ij})           & \quad & j \in \{1, \dots, n\}   \label{eqn:radius-center}\\
        &                    &       & r_{ij} \in \{0,1\}                & \quad & i,j \in \{1, \dots, n\}
    \end{alignat}
\end{subequations}

The first constraint (\ref{eqn:radius-unique}) asserts that every element is represented exactly once. The second constraint (\ref{eqn:radius-threshold}) asserts that, if $j$ represents $i$, then the dissimilarity between $i$ and $j$ is at most equal to the maximum radius. Constraint (\ref{eqn:radius-center}) asserts that, if $i$ is represented by $j$, then $j$ is a center.

The objective function (\ref{eqn:radius-objective}) is the number of clusters, to be minimized.
It is also possible to add a sub-objective minimizing the sum of dissimilarities between elements and their center. If the dissimilarities are not all zero, then the objective function becomes: 
$$
\sum_j r_{jj} + \frac{\sum_{i,j} d(i, j) \times r_{ij}}{\sum_{i,j} d(i, j)} 
$$
The second half of this objective function belongs to $[0, 1)$, ensuring that the optimal number of clusters will not be influenced.

\subsubsection*{Diameter constraint}
The maximum diameter problem can also be formulated as a constraint programming problem. Let $l_i$ be the cluster label of element $i$ and $k$ the number of clusters. 
\begin{subequations}
    \begin{alignat}{2}
     &\text{minimize } && k \\
     &\text{subject to }\  &&  l_i \neq l_j  \quad i, j \in \{1, \dots, n\}\ |\ d(i,j) > D_{max} \\
     &&& k = \max_i(l_i) \label{eqn:diameter-k}\\
     &&& l_i \in \{0, \dots, n-1\}                    \qquad\qquad  i \in \{1, \dots, n\}
    \end{alignat}
\end{subequations}

It is possible to reduce the search space by fixing labels for an independent set of elements, i.e., a set of elements that are pairwise incompatible. Since finding the maximum independent set in a graph is NP-hard, this process has to rely on heuristics to be efficient---similar to the FPF heuristic in \cite{duong2017constrained}.

\medskip

In order to transform these constrained optimization problem into integer linear programming problems, maximum equality constraints (\ref{eqn:radius-center}) and (\ref{eqn:diameter-k}) should be replaced by multiple inequalities. For instance in the radius formulation, the expression $r_{jj} = \max_i(r_{ij})$ can be translated into $\forall i : r_{jj} \geq r_{ij}$.

\subsection{Decomposition into subproblems}

Instead of considering the problem as a whole, equiwide clustering can be split into three processing steps:
\begin{enumerate}
    \item Enumerate all maximal homogeneous sets under the given threshold.
    \item Compute a minimal set cover of the population by homogeneous sets.
    \item Assign each element to a unique set of the cover.
\end{enumerate}

\subsubsection*{Enumerating homogeneous sets}

The first step of equiwide clustering is to find all maximal homogeneous sets, that is, homogeneous sets that are not proper subsets of other homogeneous sets. With a radius constraint, this task is straightforward: each element $x_i$ is successively considered as a center to compute the associated homogeneous set $S_i$. This set contains all elements that are dissimilar from $x_i$ by at most $\rmax$, i.e., $S_i=\{e\ |\ d(x_i,e) \leq \rmax \}$. This process guarantees that all maximal homogeneous sets are found, but potentially non-maximal homogeneous sets are also computed. There are at most as many homogeneous sets as the number $n$ of elements to be partitioned.

\begin{proof}[Proof that $\left\{S_1, \dots, S_n \right\}$ contains all maximal homogeneous sets]
Let $S$ be a maximal radius-homogeneous set. We show that $S \in \left\{S_1, \dots, S_n \right\}$.
Let $x_i$ be the center of $S$. The element $x_i$ belongs to $X$, so there exists $S_i \in \left\{S_1, \dots, S_n \right\} $ such that $S_i=\left\{e\ |\ d(x_i,e) \leq \rmax \right\}$.
We will show that $S=S_i$. Let $a \in S$, as $S$ is homogeneous, we have $d(x_i, a) \leq \rmax$. According to the definition of $S_i$, we then have $a \in S_i$. This proves that $S \subseteq S_i$. As $S$ is maximal and $S_i$ is homogeneous, we conclude that $S=S_i$.
\end{proof}

For a diameter constraint, enumerating all maximal homogeneous sets is much more difficult. They are the maximal cliques in the associated graph $G=(X,E)$, where $E = \{\{x_i, x_j\} \mid d(x_i, x_j) \leq \dmax\}$.
A set of vertices $C \subseteq X$ is a clique if each vertex of $C$ is adjacent to every other vertex of $C$, i.e., there is an edge between each pair of element of $C$. A clique is maximal if it does not have any strict superset which is also a clique.

\begin{proposition}
Maximal diameter-homogeneous sets are the maximal cliques of $G$.
\end{proposition}

\begin{proof}
Let us show that $A \subseteq X$ is a diameter-homogeneous set if and only if $A$ is a clique of $G$. Let $A \subseteq X$ be a diameter-homogeneous set. Then, for all pairs of elements $(x_i, x_j)$ of $A$, we have $d(x_i, x_j) \leq \dmax$ because $A$ is homogeneous. Hence, $\{x_i, x_j\}$ is an edge of $E$ and $x_i$ and $x_j$ are adjacent. Consequently, each pair of elements of $A$ are adjacent, so $A$ is a clique.
Now, if $A$ is a clique, then all pairs of elements of $A$ are adjacent, and so they are dissimilar by at most $\dmax$, hence $A$ is diameter-homogeneous.
\end{proof}

Enumerating all maximum cliques of a graph is a well-studied problem with many available algorithms and implementations, e.g. \cite{mace,networkx}. This problem is NP-hard and can generate up to $3^{n/3}$ maximal diameter-homogeneous sets \cite{moon1965cliques}.

\subsubsection*{Computing a minimal set cover}

From the set of homogeneous sets, the second step consists in identifying a minimal cover of the population. This determines the number of clusters and, for the most part, the clusters themselves. The minimal set cover is a well-known NP-hard problem \cite{karp1972} defined as follows.
Let $S$ be a set of subsets of $X$, i.e., $S \subseteq \mathbb{P}(X)$. A minimal set cover of $X$ by $S$ is a set $T \subseteq S$ such that:
\begin{align*}
    & \bigcup_{E \in T} E = X \tag{cover} \\
    & \nexists\ T' \subseteq S,\ |T'|<|T|  \wedge \bigcup_{E \in T'} E = X\tag{minimality}
\end{align*}

The resulting cover $T$ is not necessarily a valid partition since sets of the cover are not required to be pairwise disjoint. Yet, clusters will be subsets of sets in this cover.
The minimality condition guarantees that the number of clusters is minimal.
Various methods exist to solve this problem, the result can either be a single minimal solution or an enumeration of all minimal solutions.
Enumerating all solutions can be beneficial---or event required---in order to optimize the considered sub-objective, by then selecting the most appropriate cover in the enumeration.

\subsubsection*{Assigning each element to a unique cluster}

Some elements can belong to several sets in the cover computed by the previous step. These are referred to as undecided elements.
We consider approximate assignement strategies for each sub-objective. To minimize both the within cluster sum of dissimilarities and the maximum diameter, undecided elements are assigned to the cluster with the closest center. This center is computed after removal of all undecided elements.
To maximize the cluster sizes variance, a greedy algorithm is implemented to assign undecided elements to the largest clusters.
In both cases, the homogeneity of resulting clusters is trivially guaranteed for a diameter constraint. However, with a maximum radius criterion, the center (radius) of each cluster should not be removed when assigning undecided elements.

\subsection{Implementation}

The current implementation of equiwide clustering only computes comprehensive enumerations of maximal homogeneous sets. This can be prohibitive with a diameter constraint due to the the large number ($3^{n/3}$) of maximal cliques in the worst case. However, some approximation strategies exist and could be considered \cite{li2021}.

The second step computes a minimal set cover, which determines the number of clusters. We have implemented three methods to solve this problem, depending on how optimal the result should be. The first method (EQW-Exhaustive) consists in enumerating all minimal solutions \cite{shd}. The second method (EQW-LP) uses integer linear programming to find a single minimal solution \cite{vazirani2013approximation}. 
In this second implementation, variables $x_i$ denote whether set $S_i$ of $S$ should be included in the cover of $X$. The objective is to minimize the number of sets in the cover while covering every element by at least one set.
\begin{alignat*}{3}
        & \text{minimize }      &       & \sum\limits_{i | S_i \in S} x_i \\
        & \text{subject to }    & \quad & \sum\limits_{i | S_i \in S \land e\in S_i} x_i \geq 1 & \quad & e \in X \\
        &                       &       & x_i \in \{0,1\} & \quad\quad & i \ | \ S_i \in S
\end{alignat*}

The third method (EQW-Greedy) is a greedy algorithm which selects the largest clusters until all elements are covered. This algorithm finds a set cover in polynomial time, and is $H_n$-competitive with $H_n = \sum_{k=1}^n \frac{1}{k} \leq ln(n) + 1$ \cite{chvatal1979greedy}.

These three possibilities to compute the set cover impact both the optimality of the solution on the primary or secondary objectives, and the execution time.

\section{Experimental Evaluation}

\label{sec:experiments}

In order to compare the algorithms described in this paper, an experimental evaluation has been conducted. Considered algorithms from the state of the art are listed below, with their objectives summarized in Table \ref{tab:tested-alg}.
\begin{itemize}
    \item Complete-link hierarchical agglomerative clustering (HAC), using the scikit-learn implementation  \cite{scikit-learn}.
    \item Exact graph coloring, using the \textsc{dsatur}-based implementation provided by Mehrotra and Trick \cite{mehrotra1996column}.
    \item Graph coloring with minimization of the maximum diameter of clusters \cite{hansen1978complete} (denoted CG, for \textsc{clustergraph}), implemented in Python and relying on \textsc{dsatur}.
    \item Constraint programming for constrained clustering \cite{duong2017constrained} (CP4CC), using the reference implementation provided by the authors.
\end{itemize}

We compared these four algorithms to the propositions introduced in this paper, namely:
\begin{itemize}
    \item Linear programming (LP), implemented with the Coin-or branch and cut solver.
    \item Constraint programming (CP), implemented with Google's OR-Tools CP-SAT solver.
    \item All variants of equiwide clustering (EQW denotes this family of algorithms in general, while EQW-LP and EQW-Greedy identify respectively the linear programming and the greedy implementations for the computation of the minimal set cover).
\end{itemize}

\begin{table*}
\centering
\caption{Evaluated algorithms}
\label{tab:tested-alg}
\newlength{\arrskip}
\setlength{\arrskip}{0.3em}
\begin{tabular}{lclcl}
  \toprule
  \thead{Algorithm} & \thead{Minimal\\\#clusters} & \thead{Sub-objective} & \thead{Optimal} & \thead{Wideness\\constraint}  \tabularnewline \midrule
  Complete-link HCA \cite{scikit-learn}
    & no          & --                                   & --  & diameter \\\addlinespace[\arrskip]
  \textsc{dsatur} \cite{mehrotra1996column}
    & yes         & --                                   & --  & diameter \\\addlinespace[\arrskip]
  \textsc{clustergraph} \cite{hansen1978complete}
   & yes          & maximum diameter                     & yes & diameter \\\addlinespace[\arrskip]
  CP4CC \cite{duong2017constrained}
    & iteratively & \makecell[l]{maximum diameter\\WCSD} & yes & diameter \\\addlinespace[\arrskip]
  Equiwide clustering
    & yes        & \makecell[l]{variance\\WCSD}          & no  & \makecell[l]{diameter\\radius} \\\addlinespace[\arrskip]
  Integer linear programming
    & yes        & --                                    & --  & \makecell[l]{diameter\\radius} \\\addlinespace[\arrskip]
  Constraint programming
    & yes        & --                                    & --  & \makecell[l]{diameter\\radius} \\
\bottomrule
\end{tabular}
\end{table*}

\begin{table*}[t]
\sisetup{range-phrase=--}
\caption{Experimental evaluation datasets}
\label{tab:dataset}
\centering
\begin{tabular}{@{}lcS[table-format=4.1]@{}cS[table-format=4.4,table-column-width=1.6cm]S[table-format=2.1,table-column-width=1.1cm]ccS[table-format=4.2,table-column-width=1.0cm]@{}}
\toprule
  {\thead{Name}} & {\thead{Origin}} & \multicolumn{2}{c}{\thead{\#Elements}} & {\thead{\#Dimensions}} & {\thead{\#Classes}} & {\thead{Normalized (disparity)}} & {\thead{Dissimilarity}} & {\thead{D\textsubscript{max}}} \tabularnewline \midrule
  Iris & UCI & 150 & & 4 & 3 & no (2.5) & ED & 2.59e0  \tabularnewline
  Wine & UCI & 178 & & 13 & 3 & no ($2.6 \times 10^3$) & ED & 458.14 \tabularnewline
  Glass Identification & UCI & 214 & & 9 & 7 & no ($4.7 \times 10^2$) & ED & 4.98 \tabularnewline
  Ionosphere & UCI & 351 & & 34 & 2 & yes &  ED & 8.7 \tabularnewline
  User Knowledge & UCI & 403 & & 5 & 4 & yes & ED & 1.18 \tabularnewline
  WDBC & UCI & 569 & & 30 & 2  & no ($1.4 \times 10^5$) & ED & 2377.97 \tabularnewline
  Synthetic Control & UCI & 600 & & 60 & 6 & yes & ED & 109.37 \tabularnewline
  Vehicle & UCI & 846 & & 18 & 4 & no ($7.0 \times 10^1$) & ED & 264.84 \tabularnewline
  Yeast & UCI & 1484 & & 8 & 10 & yes & ED & 0.68 \tabularnewline
  Image Segmentation & UCI & 2100 & & 19 & 7 & no ($6.2 \times 10^3$) & ED & 436.5 \tabularnewline
  Waveform & UCI & 5000 & & 40 & 3 & yes & ED & 15.7 \tabularnewline
  Monthly elec. power & SRD & 145 & {($\times$ 12)} & \SIrange{4032}{4464}{} & unknown & yes & DTW & 0.1 \tabularnewline
  Yearly elec. power & SRD & 1740 & & \SIrange{4032}{4464}{} & unknown & yes & DTW & 0.1 \tabularnewline \bottomrule
\multicolumn{9}{l}{ED: Euclidean Distance} \\
\multicolumn{9}{l}{DTW: Dynamic Time Warping with a window of 4} \\
\end{tabular}
\end{table*}

These algorithms are tested on 11 datasets from the UCI Machine Learning Repository \cite{uci}. The maximum diameter threshold is taken from \cite{duong2017constrained}, where this diameter was computed by minimizing the maximum diameter of the clusters when the number of clusters is equal to the number of real classes\footnote{For image segmentation, the optimal diameter had been computed on only 2000 of the 2100 elements, resulting in a lower threshold.}. In these experiments, thresholds are rounded up to avoid round-off error. The considered dissimilarity is the Euclidean distance. Half (five) of these datasets were normalized by their authors, meaning that the ranges of their attributes are comparable, either inherently or as a post-processing step. The other half (six) were not normalized, the weight of their attributes can therefore vary significantly---by a factor of up to 140,000 for WDBC---when computing the Euclidean distance.
Table \ref{tab:dataset} provides a summary of the datasets used in this experimental evaluation. When the dataset is not normalized, an indication of the disparity between attributes is provided as the ratio between the largest and smallest ranges.

Algorithms are also tested on 12+1 datasets provided by SRD.
Each element in these datasets is a time series of electric power passing through one of the 145 busbar output of step-down substations\footnote{Out of the 180 available, 35 outputs were not included in this study due to maintenance operations over the period and other practical considerations.}.
These time series are normalized with respect to the total power subscription of connected clients.
Each of the 12 datasets corresponds to a month of 2017 with a sampling period of ten minutes. A 13th dataset is constructed as the union of the previous 12, therefore containing $145 \times 12$ time series\footnote{Series of different length were truncated during the computation of the DTW dissimilarity to accommodate for months with fewer than 31 days.}. 
Each dataset is converted into a dissimilarity matrix using Dynamic Time Warping (DTW) with a window of 4, which serves as the input of clustering algorithms. The computation time of this dissimilarity matrix is not taken into account in these experiments.
The maximum diameter threshold is 0.1, which has been determined by network operators at SRD as an acceptable margin.

Each algorithm is executed on each dataset with four randomized ordering of elements. For algorithms with a radius constraint, the maximum radius is defined as half the value of the maximum diameter.
Experiments are performed on an 2.40~GHz Intel Xeon CPU E5-2630 processor, with 32~GB of RAM, running Ubuntu 18.04. The Python interpreter is CPython 3.8.0. Each test has a maximum run time of ten minutes.
None of the algorithms is parallelized.

\medskip

\begin{table*}
\sisetup{range-phrase=--}
\robustify\bfseries
\sisetup{detect-weight, range-units=single}
\centering
\caption{Number of clusters}
\label{tab:nb_clust_uci}
\begin{tabular}{lr@{}lS[table-format=3.0]r@{}lr@{}lS[table-format=2.0]r@{}lcS[table-format=2.0]}
  \toprule
   & \multicolumn{3}{c}{\thead{Radius}} & \multicolumn{5}{c}{\thead{Diameter}} & \multicolumn{4}{c}{\thead{1.2 $\times$ Diameter}} \tabularnewline 
   \cmidrule(r){2-4}
   \cmidrule(lr){5-9}
   \cmidrule(l){10-13}
  \thead{Dataset} & \multicolumn{2}{c}{\makecell{EQW-Greedy}} & {Opt.} & \multicolumn{2}{c}{\makecell{EQW-Gr.}} & \multicolumn{2}{c}{HAC} & {Opt.} & \multicolumn{2}{c}{\makecell{EQW-Gr.}} & HAC & {Opt.} \tabularnewline\midrule
  Iris                 &      \bfseries 4 &       & \bfseries   4 &   \hspace{1em} 4 &     &           4 &      & \bfseries   3 &   \hspace{1em} 4 &     &           4 & \bfseries   3 \tabularnewline
  Wine                 &      \bfseries 4 &   --5 & \bfseries   4 &                4 &     &           4 &      & \bfseries   3 &      \bfseries 3 & --4 &           4 & \bfseries   3 \tabularnewline
  Glass Identification &     \bfseries 13 &  --14 & \bfseries  13 &                8 & --9 &          10 &      & \bfseries   7 &      \bfseries 6 &     &           7 & \bfseries   6 \tabularnewline
  Ionosphere           &     \bfseries 28 &  --29 & \bfseries  28 &      \bfseries 2 & --3 &           4 &      & \bfseries   2 &      \bfseries 1 &     & \bfseries 1 & \bfseries   1 \tabularnewline
  User knowledge       & \hspace{1.5em} 9 &  --11 & \bfseries   8 &               -- &     &           7 &      & \bfseries   4 &      \bfseries 2 & --3 & \bfseries 2 & \bfseries   2 \tabularnewline
  WDBC                 &      \bfseries 3 &       & \bfseries   3 &      \bfseries 2 &     &           3 &      & \bfseries   2 &      \bfseries 2 &     & \bfseries 2 & \bfseries   2 \tabularnewline
  Synthetic Control    &               16 &  --17 & \bfseries  14 &               -- &     &           9 &      & \bfseries   6 &      \bfseries 4 & --6 &           5 & \bfseries   4 \tabularnewline
  Vehicle              &                6 &   --7 & \bfseries   5 &                5 &     &           6 &      & \bfseries   4 &      \bfseries 4 &     &           5 & \bfseries   4 \tabularnewline
  Yeast                &               21 &       & \bfseries  18 &               -- &     &          14 & --15 & \bfseries  10 &                8 & --9 &           9 & \bfseries   7 \tabularnewline
  Image Segmentation   &               13 &       & \bfseries  12 &      \bfseries 8 &     & \bfseries 8 &      & \bfseries   8 &      \bfseries 5 &     & \bfseries 5 & \bfseries   5 \tabularnewline
  Waveform             &              168 & --172 & \bfseries 149 &               -- &     &           6 &      & \bfseries   3 &               -- &     &           3 & \bfseries   2 \tabularnewline
  Elec. power (July)   &     \bfseries 13 & --14  & \bfseries  13 &                9 &     &          10 &      & \bfseries   8 \tabularnewline
  Yearly elec. power   &     \bfseries 16 &       & \bfseries  16 &               -- &     &          30 &      & \bfseries  19 \tabularnewline \bottomrule
  \multicolumn{12}{l}{Execution times exceeding the 10 minutes limit are denoted by --} \\
\end{tabular}
\end{table*}

\subsubsection*{Number of clusters}
We compare separately algorithms allowing a diameter constraint and algorithms allowing a radius constraint, but the same criteria of quality apply. First, every clustering has to satisfy the provided wideness constraint, which is the case in every experiment. Then, the main quality criterion is the number of resulting clusters, as the objective is to minimize this value. Experimental results are displayed in Table \ref{tab:nb_clust_uci} (columns Radius and Diameter). In this table, the Opt. column refers to all algorithms which gave the same optimal number of clusters.
When two values are provided, they represent the minimum and the maximum number of clusters among the different orderings.
As expected by their design, every algorithm except for HAC and the greedy variant of EQW provides the optimal number of cluster.

Since the specified diameter threshold is tight for the UCI datasets, we also experimented with a threshold set to 1.2 times the optimal diameter, in order to evaluate the impact of having relaxed constraints and potentially a larger solution space. Results are presented in Table \ref{tab:nb_clust_uci} (column 1.2 $\times$ Diameter). With these new thresholds, results from both approximate algorithms become closer to optimal values.

Overall, the greedy version of EQW usually provides slightly fewer clusters than HAC, with some guarantees related to its $H_n$-competitiveness.
For instance, with the ionosphere dataset, the greedy version of EQW guarantees to find a solution with at most $\lfloor2 H_2\rfloor = 3$ clusters, while HAC provides a solution with 4 clusters.

\subsubsection*{Execution time}
With diameter constraints, the constraint programming implementation significantly outperforms the linear programming implementation. On the opposite, with radius constraints, the linear programming implementation outperforms the constraint programming one. Only results for the best of these two baseline implementations are presented.

On all datasets except for iris, wine and monthly electric power, the variant of EQW that performs an exhaustive enumeration of minimal set covers fails to provide a solution within 10 minutes. Its results are therefore omitted from this analysis.
With diameter constraints, the predominance of the first step---the enumeration of homogeneous sets---in the execution time causes all variants of EQW to be within a few percent of each other. Their execution times are consequently combined under the EQW label.
With radius constraints, the second step---the computation of a minimal set cover---becomes predominant and variants of EQW are then split between the EQW-Greedy and the EQW-LP labels.

The execution time on the monthly electric power datasets being similar, only a representative month, July, is reported in the results.
Execution times are reported in Table \ref{tab:time_diameter} for diameter constraints, and Table \ref{tab:time_radius} for radius constraints.

\begin{table*}
\centering
\caption{Execution time (in seconds) for algorithms with a diameter constraint}
\label{tab:time_diameter}
\begin{tabular}{lr@{$\,$}lr@{$\,$}lr@{$\,$}lr@{$\,$}lr@{$\,$}lr@{$\,$}l}
\toprule
\thead{Dataset} &
\multicolumn{2}{c}{\thead{HAC}}    &
\multicolumn{2}{c}{\thead{DSATUR}} &
\multicolumn{2}{c}{\thead{CG}}     &
\multicolumn{2}{c}{\thead{CP4CC}}  &
\multicolumn{2}{c}{\thead{CP}}     &
\multicolumn{2}{c}{\thead{EQW}}    \\
\midrule
                Iris & \multicolumn{2}{c}{$< 0.01$} &   0.07 & $\pm\,$0.0 &   0.14 & $\pm\,$0.0 &    0.11 & $\pm\,$0.0 &    0.47 & $\pm\,$0.0 &   0.07 & $\pm\,$0.1 \\
                Wine & \multicolumn{2}{c}{$< 0.01$} &   0.08 & $\pm\,$0.0 &   0.15 & $\pm\,$0.0 &    0.20 & $\pm\,$0.0 &    0.44 & $\pm\,$0.0 &   0.10 & $\pm\,$0.0 \\
Glass Identification & \multicolumn{2}{c}{$< 0.01$} &   0.09 & $\pm\,$0.0 &   0.18 & $\pm\,$0.0 &    0.52 & $\pm\,$0.0 &    0.26 & $\pm\,$0.0 &   0.12 & $\pm\,$0.1 \\
          Ionosphere & \multicolumn{2}{c}{$< 0.01$} &   0.08 & $\pm\,$0.0 &   0.10 & $\pm\,$0.0 &    0.74 & $\pm\,$0.0 &    0.02 & $\pm\,$0.0 &   0.05 & $\pm\,$0.0 \\
      User Knowledge & \multicolumn{2}{c}{$< 0.01$} &   0.13 & $\pm\,$0.0 &   0.22 & $\pm\,$0.0 &    0.85 & $\pm\,$0.0 &    0.22 & $\pm\,$0.0 &     -- &            \\
                WDBC &            0.01 & $\pm\,$0.0 &   0.29 & $\pm\,$0.0 &   0.42 & $\pm\,$0.0 &    1.54 & $\pm\,$0.1 &    0.55 & $\pm\,$0.1 &   0.14 & $\pm\,$0.0 \\
   Synthetic Control &            0.01 & $\pm\,$0.0 &   0.62 & $\pm\,$0.1 &   1.07 & $\pm\,$0.0 &    8.29 & $\pm\,$0.0 &   11.44 & $\pm\,$6.7 &     -- &            \\
             Vehicle &            0.02 & $\pm\,$0.0 &   1.08 & $\pm\,$0.0 &   1.76 & $\pm\,$0.0 &    5.49 & $\pm\,$0.1 &   23.50 & $\pm\,$0.7 &  45.46 & $\pm\,$3.7 \\
               Yeast &            0.06 & $\pm\,$0.0 &   2.29 & $\pm\,$0.0 &   4.83 & $\pm\,$3.9 &   29.87 & $\pm\,$0.1 &   10.11 & $\pm\,$1.1 &     -- &            \\
  Image Segmentation &            0.13 & $\pm\,$0.0 &   4.64 & $\pm\,$0.1 &   9.17 & $\pm\,$1.1 &   70.68 & $\pm\,$1.5 &    4.68 & $\pm\,$0.6 &   1.31 & $\pm\,$0.1 \\
            Waveform &            0.91 & $\pm\,$0.0 &  32.68 & $\pm\,$1.7 &  41.38 & $\pm\,$2.6 &  239.98 & $\pm\,$6.8 &  483.65 & $\pm\,$23  &     -- &            \\
  Elec. power (July) & \multicolumn{2}{c}{$< 0.01$} &   0.03 & $\pm\,$0.0 &   0.07 & $\pm\,$0.0 &   11.58 & $\pm\,$TL  &    0.14 & $\pm\,$0.0 &   0.04 & $\pm\,$0.0 \\
  Yearly elec. power &            0.10 & $\pm\,$0.0 &   3.87 & $\pm\,$TL  &     -- &            &      -- &            &   55.60 & $\pm\,$TL  &     -- &            \\
\bottomrule
\multicolumn{13}{l}{$\pm\,$TL indicates that the 10 minutes limit has been exceeded in at least one of the four runs}
\end{tabular}
\end{table*}

\begin{table}
\centering
\caption{Execution time for algorithms with a radius constraint}
\label{tab:time_radius}

\begin{tabular*}{\columnwidth}{@{}lr@{$\,$}l@{$\ $}lR@{$\,$}l@{$\ $}lR@{$\,$}l@{$\ $}l@{}}
\toprule
\thead{Dataset} &           \multicolumn{3}{c@{\hspace{2em}}}{\thead{EQW-Greedy}} &            \multicolumn{3}{c@{\hspace{2em}}}{\thead{EQW-LP}} &        \multicolumn{3}{c}{\thead{LP}}\\
\midrule
                Iris &  5.43 & $\pm\,$0.56 & ms & 45.51 & $\pm\,$3.25 & ms &   0.78 & $\pm\,$0.04 & s \\
                Wine &  7.97 & $\pm\,$0.55 & ms & 50.30 & $\pm\,$1.57 & ms &   2.28 & $\pm\,$1.02 & s \\
Glass Identification & 14.61 & $\pm\,$0.82 & ms & 67.35 & $\pm\,$5.07 & ms &   1.77 & $\pm\,$0.10 & s \\
          Ionosphere & 55.78 & $\pm\,$1.65 & ms &  0.17 & $\pm\,$0.02 & s  &   6.60 & $\pm\,$1.09 & s \\
      User Knowledge & 46.74 & $\pm\,$0.96 & ms &  0.24 & $\pm\,$0.01 & s  &  18.12 & $\pm\,$10.0 & s \\
                WDBC &  0.21 & $\pm\,$0.03 & s  &  0.34 & $\pm\,$0.02 & s  & 157.47 & $\pm\,$111  & s \\
   Synthetic Control & 74.58 & $\pm\,$5.00 & ms &  0.23 & $\pm\,$0.01 & s  &  26.28 & $\pm\,$11.2 & s \\
             Vehicle &  0.37 & $\pm\,$0.01 & s  &  0.80 & $\pm\,$0.01 & s  &  78.36 & $\pm\,$130  & s \\
               Yeast &  2.43 & $\pm\,$0.04 & s  &  7.34 & $\pm\,$0.13 & s  &     -- &             &   \\
  Image Segmentation & 11.27 & $\pm\,$0.08 & s  & 16.26 & $\pm\,$0.27 & s  &     -- &             &   \\
            Waveform &  7.69 & $\pm\,$0.21 & s  &   --  &             &    &     -- &             &   \\
  Elec. power (July) &  8.05 & $\pm\,$0.44 & ms & 55.76 & $\pm\,$3.58 & ms &   0.86 & $\pm\,$0.01 & s \\
  Yearly elec. power &  6.64 & $\pm\,$0.10 & s  & 11.74 & $\pm\,$0.09 & s  &     -- &             &   \\
\bottomrule
\end{tabular*}
\end{table}

\medskip

HAC is by far the fastest algorithm, but it rarely reaches the minimal number of clusters. The next fastest algorithm is \textsc{dsatur}, followed by \textsc{clustergraph}. The former is always faster than the latter since \textsc{clustergraph} makes multiple calls---at least two---to \textsc{dsatur}.

In the reference implementation of CP4CC, the FPF heuristic can not be used with a dissimilarity matrix as the input. Consequently, on the datasets provided by SRD, this heuristic is not used and CP4CC exceeds the time limit of 10 minutes on 26 of the 48 experiments performed with the monthly electric power datasets, despite them being relatively simple compared with other datasets provided by UCI. This behavior highlights the influence of heuristics to solve the equiwide clustering problem.
The FPF reordering also allows CP4CC to be independent on the innate ordering of the input data.
Otherwise, exact algorithms---\textsc{DSATUR}, CP and CP4CC without the FPF heuristic---are all dependent on this ordering, as most notably seen on the SRD dataset, where there is at least a tenfold difference between the fastest and slowest run.

EQW does not follow the performance pattern of other algorithms, for which the execution time is mostly related to the number of elements in the dataset. This is explained by the potentially large number of homogeneous sets that have to be computed. For instance, EQW could not complete on User Knowledge which only contains 403 elements, and for which there can theoretically be up to $1.2 \times 10^{64}$ maximal homogeneous sets. In fact, after ten minutes, EQW already identified $3.9 \times 10^{7}$ maximal homogeneous sets. The main issue in addressing this problem is that it is not possible to efficiently predict this behavior from the dissimilarity matrix.
This limitation on how EQW computes homogeneous sets with a diameter constraint severely limits its ability to provide exact solutions in a reasonable time compared to approaches from the state of the art. It then seems only appropriate to use EQW if this step can be approximated upon reaching an excessive number of homogeneous sets.

\medskip

The limitation on homogeneous sets is not present with radius constraints. EQW then becomes a competitive solution, the baseline linear programming and constraint programming implementations being comparatively much slower.
The greedy variant of EQW does not provide a minimal number of clusters, but is applicable on every considered dataset since its overall complexity is polynomial. 
Alternatively, potentially better approximate solutions can be provided with a slight modification of the linear programming implementation of EQW. For example with the waveform dataset, EQW-LP fails to provide an optimal result within 10 minutes. However, the underlying solver allows for a time limit to be set and, with a one minute limit, finds an approximate solution with 150 clusters---instead of the optimal value of 149---that still improves upon the 168 clusters solution identified by the greedy variant.

\section{Conclusion}

In this work, we provide an homogeneous vision of two problems from graph theory under a common clustering perspective. This leads to the definition of the equiwide clustering problem.
The underlying dissimilarity measure becomes a core component that must be intelligible, as a threshold has to be provided according to the application needs.
Given the importance of both the dissimilarity and the threshold, a semi-supervised approach could be considered to assist the user in this task, for instance using metric learning \cite{bilenko2004}.

For distance-based intra-cluster constraints, graph coloring-based algorithms appear to be the most promising solutions from the state of the art. For larger datasets, the complete-link agglomerative hierarchical clustering provides a valid yet non-optimal solution in a much shorter time. EQW, in both its exact and approximate variants, is not competitive in this setting.

The extent to which sub-objectives can be integrated into minimal cardinality dominating set computation algorithms remains to be determined in order to assess their applicability as clustering algorithms for radius constraints.
Yet, EQW fills this gap and offers exact or approximate solutions with acceptable performance when a representative element for each cluster has to be selected from the original dataset.

\section*{Acknowledgment}

This work was supported by Région Nouvelle-Aquitaine and the aLIENOR ANR LabCom (ANR-19-LCV2-0006).

\bibliography{eqwclustering}

\end{document}